\documentclass{article}

\PassOptionsToPackage{numbers, compress}{natbib}

\usepackage[preprint]{neurips_2026}

\usepackage[utf8]{inputenc} 
\usepackage[T1]{fontenc}    
\usepackage{hyperref}       
\usepackage{url}            
\usepackage{booktabs}       
\usepackage{amsfonts}       
\usepackage{nicefrac}       
\usepackage{microtype}      
\usepackage{xcolor}         
\usepackage{graphicx}       
\usepackage{algorithm}
\usepackage{algorithmicx}
\usepackage{amsmath, amssymb, amsthm}
\newtheorem{proposition}{Proposition}
\newtheorem{lemma}{Lemma}
\newtheorem{theorem}{Theorem}
\newtheorem{corollary}{Corollary}
\usepackage{algpseudocode}
\usepackage{kotex}
\usepackage{cleveref}
\usepackage{subcaption}
\usepackage{wrapfig}
\usepackage{multirow}
\usepackage{colortbl}
\usepackage{comment}
\usepackage{svg}
\usepackage{pifont}
\svgsetup{inkscapelatex=false}

\title{AdaHOP: Fast and Accurate Low-Precision Training \\
via Outlier-Pattern-Aware Rotation}

\author{%
  Seonggon Kim$^{1,2}$\thanks{Work done during internship at Advanced Micro Devices, Inc.} \quad Alireza Khodamoradi$^{1}$ \quad Pranathi Vasireddy$^{1}$ \\[0.5em] \quad \textbf{Kristof Denolf}$^{1}$ \quad \textbf{Eunhyeok Park}$^{2}$
  \\[1em]
  $^{1}$Advanced Micro Devices, Inc. \quad $^{2}$POSTECH \\ [1em]
  \texttt{\{sungonuni, eh.park\}@postech.ac.kr} \\ \texttt{\{alireza.khodamoradi, Pranathi.Vasireddy, kristof.denolf\}@amd.com}
}

\begin{document}

\maketitle

\begin{abstract}
Hadamard transforms have become a key tool for stabilizing low-precision training, but existing methods apply them uniformly across tensors and computation paths. We show that this one-size-fits-all strategy is inherently limited: Hadamard transformation reduces quantization error only when its direction is properly aligned with the operand’s outlier structure. Through a systematic study of weights, activations, and gradients in LLM training, we identify three stable outlier patterns, Row-wise, Column-wise, and None, and show that each outlier pattern pair in matrix multiplication requires a distinct transform or outlier-handling strategy. We propose AdaHOP, \textbf{Ada}ptive \textbf{H}adamard transform with \textbf{O}utlier-\textbf{P}attern-aware strategy, which applies Inner Hadamard Transform~(IHT) when inner-dimension mixing properly suppresses the operands’ outliers, and selectively applies Outlier Extraction~(OE) that extracts dominant outlier rows or columns into a high-precision path when it does not. With fused, hardware-aware Triton kernels, AdaHOP enables training from scratch at MXFP4 with BF16-level quality, while achieving up to 3.6$\times$ reduced memory usage and a 1.46$\times$ end-to-end training speedup over BF16.

\end{abstract}

\section{Introduction}
\label{sec:introduction}
Model optimization has been widely studied for efficient LLM inference, but it has been adopted more cautiously for training. Since training was traditionally viewed as a one-time cost, the potential quality risks introduced by low-precision arithmetic were often harder to justify. This assumption is becoming less tenable as LLM training costs continue to grow, model lifecycles shorten, and new models must be trained more frequently. Reducing training cost is therefore an important challenge.

Low-Precision Training (LPT) addresses this challenge by reducing precision from BF16 to hardware-acceleratable formats such as FP8 or MXFP4~\cite{chen2025int}, lowering memory consumption and improving throughput. However, stable LPT remains difficult because outliers can dominate the quantization range, amplify rounding errors, and degrade training quality. Recent methods mitigate this issue with Hadamard transforms, which redistribute outlier energy before quantization~\cite{QuaRot, SpinQuant, HALO}. Yet, these methods typically apply a fixed transform uniformly across tensors or computation paths, implicitly assuming that a single smoothing direction is broadly effective.

We show that this assumption often breaks down: Hadamard smoothing is effective only when its smoothing direction aligns with the outlier structure of the operands. Through a systematic analysis of weights, activations, and gradients, we identify three recurring spatial outlier patterns: \textbf{Row-wise}, \textbf{Column-wise}, and \textbf{None}. These patterns vary across tensor types and computation paths, but remain stable during training, enabling calibration-based strategy selection.

Based on this insight, we propose AdaHOP, an adaptive outlier-pattern-aware framework for MXFP4 training. AdaHOP uses Inner Hadamard Transform (IHT) as an efficient default when inner-dimension smoothing is aligned with the outlier patterns. When IHT is insufficient, AdaHOP applies selective Outlier Extraction (OE), routing only dominant outlier rows or columns to a small high-precision path while keeping the residual computation in MXFP4. Thus, AdaHOP turns Hadamard-based LPT from a fixed transform rule into a pattern-aware strategy selection problem.

AdaHOP is implemented with fused Triton kernels on AMD CDNA4\footnote{The CDNA4 architecture is utilized on AMD Instinct\textsuperscript{TM} 
MI350 Series GPUs.} architectures. By executing the small BF16 outlier path alongside the dominant MXFP4 residual path, AdaHOP obtains the accuracy benefit of OE with low marginal cost. It achieves BF16-level training quality at MXFP4 precision, with up to 3.6$\times$ memory savings, a 2.1$\times$ kernel-level and a 1.46$\times$ end-to-end speedup over BF16.

Our contributions are summarized as follows:

\textbf{Outlier-pattern analysis.}
We present a systematic analysis of outlier structures in LLM training tensors and identify stable spatial patterns across weights, activations, and gradients.

\textbf{Pattern-aware LPT algorithm.}
We show that Hadamard effectiveness depends on outlier pattern pairs and propose AdaHOP, which uses selective optimization depending on outlier patterns.

\textbf{Hardware-efficient implementation.}
We implement AdaHOP with fused Triton kernels on AMD CDNA4, exploiting mixed-precision execution to achieve BF16-level quality with MXFP4 efficiency.

\begin{figure*}[t]
    \centering
    \includesvg[width=\textwidth]{figs/loss_analysis_2x2.svg}
    \caption{Training loss curves and loss difference relative to BF16 for (Top) Llama3.2-1B and (Bottom) Instella-3B. AdaHOP consistently achieves the lowest loss gap relative to BF16 among all MXFP4-based methods.}
    \vspace{-8pt}
    \label{fig:loss_analysis}
\end{figure*}

\section{Related Work}
\label{sec:related_work}

\subsection{Quantization and Low-Precision Training}

Reducing numerical precision is a central approach to improving LLM training efficiency, with formats ranging from FP8~\cite{FP8-LM, COAT, Jetfire, micikevicius2022fp8} to sub-8-bit formats such as INT4~\cite{NITI, Accurate, LUQ, schiemer2023hadamard} and MXFP4~\cite{Albert, BRIDGING, FP4-All-The-Way, Quartet}. These formats reduce memory footprint and improve throughput, but lower bit-widths also make training more sensitive to quantization error. A major source of this error is outliers: a few extreme values can dominate the quantization range and amplify rounding errors for non-outlier values, leading to loss spikes and degraded convergence~\cite{LLM.int8(), Smoothquant}.

Several techniques improve the stability of low-precision arithmetic, including mixed-precision training~\cite{micikevicius2018mixed}, loss scaling~\cite{micikevicius2018mixed}, and stochastic rounding~\cite{gupta2015deep, LUQ}. However, these methods do not directly address the structural outliers that distort the tensor dynamic range. Outlier-aware scaling methods, such as channel-wise or group-wise scaling~\cite{Smoothquant, AWQ}, tackle this issue more directly by adapting the quantization range to local tensor statistics. Recent rotation-based methods, which we discuss next, offer a complementary approach by suppressing outliers through energy redistribution.

\subsection{Outlier Suppression via Rotation}
\label{sec:outlier_suppression_via_rotation}

Outliers in LLM tensors, especially activations, are a well-known obstacle to accurate quantization~\cite{Smoothquant, LLM.int8(), AWQ}. For inference, methods such as QuaRot, SpinQuant, QuIP, QuIP\#, and FlatQuant~\cite{QuaRot, SpinQuant, QUIP, QUIPsharp, Flatquant} apply Hadamard or learned rotations to spread outlier energy across channels. Extending this idea to training is more difficult: in addition to weights and activations, gradients introduce a third tensor type, and quantization must remain stable across the forward and backward computation paths.

Recent LPT methods incorporate Hadamard transforms into training. HOT~\cite{HOT} uses Hadamard quantization selectively on gradient paths through path-level heuristics. Tseng et al.~\cite{Albert} combine random Hadamard transforms with stochastic rounding to obtain unbiased MXFP4 gradient estimates. Quartet~\cite{Quartet} extends this direction to native MXFP4 training by using fixed Hadamard transforms in the forward pass and randomized Hadamard transforms in
the backward pass. HALO~\cite{HALO} further observes that gradients and activations often exhibit row-wise and column-wise outliers, respectively, and places Hadamard rotations accordingly.

Despite these advances, existing methods still rely on fixed transform assignments at the tensor or path level. They do not account for the fact that outlier structures vary across operand pairs within different matrix multiplications. AdaHOP differs from these methods by moving the unit of decision from the tensor/path level to the operand-pair level.
\section{Preliminary: Low-Precision Training and Inner Hadamard Transformation}
\label{sec:preliminary}

Matrix multiplication dominates the computation of LLM training and is therefore a major target for optimization. For a linear layer, the forward and backward matrix multiplications are
\begin{align}
    \text{Forward:} \quad & Y = X W^\top, \label{eq:forward} \\
    \text{Backward ($\nabla W$):} \quad & G_W = G_Y^\top X, \label{eq:backward_gw} \\
    \text{Backward ($\nabla X$):} \quad & G_X = G_Y W, \label{eq:backward_gx}
\end{align}
where $X \in \mathbb{R}^{b \times d_{\text{in}}}$ is the input activation, $W \in \mathbb{R}^{d_{\text{out}} \times d_{\text{in}}}$ is the weight matrix, $b$ is the number of tokens, $d_{\text{in}}$ and $d_{\text{out}}$ are the input and output dimensions, and $G_Y \in \mathbb{R}^{b \times d_{\text{out}}}$ denotes the gradient of the loss with respect to $Y$.

LPT improves efficiency primarily by quantizing matrix-multiplication operands. For a generic product $C = AB$, where $A \in \mathbb{R}^{m \times k}$ and $B \in \mathbb{R}^{k \times n}$, LPT computes
\begin{equation}
    C \approx Q(A)Q(B),
\end{equation}
where $Q(\cdot)$ denotes the quantization operator. This introduces perturbations that can accumulate and degrade convergence. A major error source is outliers: a few large values dominate the quantization range and amplify rounding errors. Suppressing outliers before quantization is therefore crucial.

A widely used outlier-suppression technique is the Inner Hadamard Transform (IHT), which applies a Hadamard rotation along the shared dimension of a matrix multiplication. Let $H_k \in \mathbb{R}^{k \times k}$ be a normalized Walsh--Hadamard matrix satisfying $H_k^\top H_k = I_k$. IHT computes
\begin{equation}
    C_{\text{IHT}} = Q(AH_k) \cdot Q(H_k^\top B).
    \label{eq:inner_hadamard}
\end{equation}
In infinite precision, this transformation preserves the product exactly since $(AH_k)(H_k^\top B) = AB$. Under quantization, its benefit comes from mixing values along the shared dimension before quantization: $AH_k$ smooths the columns of $A$, while $H_k^\top B$ smooths the rows of $B$, which can reduce quantization error when the smoothing direction is aligned with the operands' outlier structures~\cite{HOT, yang2023efficient}.

A complementary rotation is the Outer Hadamard Transform (OHT):
\begin{equation}
    C_{\mathrm{OHT}}
    =
    H_m^\top \bigl(Q(H_m A) \cdot Q(BH_n)\bigr) H_n^\top .
    \label{eq:outer_hadamard}
\end{equation}
OHT smooths the rows of $A$ and columns of $B$, but requires outer-dimension transforms and inverse rotations after multiplication, introducing extra data movement. OHT can be useful for compatible patterns but is less efficient as a default strategy.
The key question is therefore not whether to use Hadamard transforms, but which dimension should be smoothed for each operand pair. We answer this question by analyzing outlier patterns in training.

\section{Analysis of Outlier Patterns}
\label{sec:analysis}

We now examine the outlier structures underlying the alignment issue discussed above. Rather than treating outliers as isolated large values, we analyze how they are spatially organized within LLM training tensors. This perspective allows us to reason about why the same Hadamard transform can reduce quantization error in some computation paths but fail in others.

\begin{figure*}[t]
    \centering
    \includesvg[width=\textwidth]{figs/3d_plot_samples.svg}
    \caption{(Left) 3D visualization of weight, activation, and gradient tensors from Llama3.2-1B's 
    \texttt{block.2.self\_attn\_o}. (Right) A column-wise outlier tensor after different Hadamard directions. Right Hadamard suppresses column-wise outliers; Left Hadamard leaves the structure intact.}
    \vspace{-8pt}
    \label{fig:3d_plot}
\end{figure*}

\subsection{Definition of Outlier Patterns}
\label{sec:outlier_pattern_def}
We classify outlier structures by the dimension along which large magnitudes are concentrated. Specifically, we define three patterns, detected using the normalized Coefficient of Variation (CV) metric described in \Cref{sec:appendix_calibration}:

\begin{itemize}
    \item \textbf{Row-wise (R)}: Outliers are concentrated in a small number of rows; i.e., certain rows have significantly larger magnitudes than others.
    \item \textbf{Column-wise (C)}: Outliers are concentrated in a small number of columns; i.e., certain columns have significantly larger magnitudes than
    others.
    \item \textbf{None (N)}: Magnitudes are relatively uniform across rows and columns, with no pronounced row- or column-wise concentration.
\end{itemize}

\Cref{fig:3d_plot} (Left) shows an early block (\texttt{block.2.self\_attn\_o}) of Llama3.2-1B after 100 training steps on C4~\cite{C4}. The weight tensor is classified as None, showing no pronounced directional concentration. The activation tensor exhibits column-wise outliers, where a small subset of input channel dimensions remains consistently large across token positions. In contrast, the gradient tensor exhibits row-wise outliers, where a small number of token positions carry consistently large magnitudes across output channel dimensions. However, these examples are illustrative rather than exhaustive: a tensor type is not restricted to a single pattern. In particular, gradient tensors may exhibit Column-wise or Row-wise depending on the projection type (\Cref{fig:op_per_step}). The key point is that the effect of a Hadamard transform depends not only on the presence of outliers, but also on the dimension along which they are concentrated. The next subsection makes this dependence explicit.

\subsection{Why Outlier Patterns Matter}
\label{sec:why_patterns_matter}

The effectiveness of a Hadamard transform is determined by how its smoothing direction aligns with the outlier structure. \Cref{fig:3d_plot} (Right) illustrates this effect for a tensor with column-wise outliers. A right-sided Hadamard transform, $XH$, mixes values across columns and suppresses the column-wise concentration, whereas a left-sided transform, $HX$, mixes across rows and leaves the outlier structure largely intact. Thus, outlier suppression requires smoothing along the outlier-concentrated dimension.

This directional effect directly extends to matrix multiplication. In \Cref{eq:inner_hadamard}, IHT smooths the columns of $A$ and the rows of $B$. It is highly effective for the \textbf{CR} pair, where both operands are smoothed along their outlier dimensions. In contrast, for the \textbf{RC} pair, both smoothing directions are misaligned with the outlier structures, making IHT highly ineffective.

This directional dependence explains why a single fixed transform cannot be optimal for all computation paths. Existing methods still use fixed assignments: Tseng et al.~\cite{Albert} applies IHT uniformly, while HALO~\cite{HALO} applies OHT uniformly across $\nabla X$ paths. However, IHT and OHT are each effective only for compatible outlier patterns; for incompatible pairs, they may leave dominant outliers unsmoothed or even increase quantization error. We therefore analyze this at the level of outlier pattern pairs.

We empirically validate this alignment effect in \Cref{fig:mse_analysis}, which measures the reduction in matrix-multiplication error after applying IHT to each outlier pattern pair. Across both synthetic tensors and real LLM tensors, IHT improves only the pairs whose outlier structures are compatible with its smoothing directions. The largest gain appears for \textbf{CR}, while \textbf{CN} and \textbf{CC} receive partial benefits; for the remaining pairs, IHT is largely ineffective or even harmful.

Thus, IHT is not universally beneficial. Its effectiveness is governed by the alignment between Hadamard smoothing directions and operand outlier structures, which explains why a uniform IHT strategy fails to minimize quantization error across all computation paths.

\subsection{Outlier Patterns in LLM Layers}
\label{sec:patterns_in_layers}

We next examine which pattern pairs actually arise during LLM training. Using the CV-based detection method, we measure pattern-pair frequencies across Llama3.2-1B, Instella-3B, and Llama3.1-8B after 30 training steps.
\Cref{tab:pattern_distribution} summarizes the results.

\begin{table*}[t]
\centering
\caption{Distribution of outlier pattern pairs across computation paths for three LLM architectures after 30 training steps on C4. Patterns RR, NR, and CR do not appear because weights consistently exhibit None, and activations exhibit Column-wise or None.}
\vspace{-6pt}
\label{tab:pattern_distribution}
\small
\begin{tabular}{l rrrr rrrr rrrr}
\toprule
 & \multicolumn{4}{c}{\textbf{Llama3.2-1B}} & \multicolumn{4}{c}{\textbf{Instella-3B}} & \multicolumn{4}{c}{\textbf{Llama3.1-8B}} \\
\cmidrule(lr){2-5} \cmidrule(lr){6-9} \cmidrule(lr){10-13}
\textbf{Pattern Pair} & Total & Fwd & $\nabla W$ & $\nabla X$ & Total & Fwd & $\nabla W$ & $\nabla X$ & Total & Fwd & $\nabla W$ & $\nabla X$ \\
\midrule
RN & 23 & 0 & 15 & 8 & 79 & 0 & 33 & 46 & 65 & 0 & 27 & 38 \\
RC & 69 & 0 & 69 & 0 & 111 & 0 & 111 & 0 & 92 & 0 & 92 & 0 \\
NN & 35 & 15 & 0 & 20 & 100 & 36 & 2 & 62 & 103 & 32 & 4 & 67 \\
NC & 20 & 0 & 20 & 0 & 60 & 0 & 60 & 0 & 63 & 0 & 63 & 0 \\
CN & 181 & 97 & 0 & 84 & 361 & 216 & 1 & 144 & 312 & 192 & 1 & 119 \\
CC & 8 & 0 & 8 & 0 & 45 & 0 & 45 & 0 & 37 & 0 & 37 & 0 \\
\bottomrule
\end{tabular}
\end{table*}

\begin{figure*}[t]
    \centering
    \includesvg[width=\textwidth]{figs/mse_analysis.svg}
    \vspace{-10pt}
    \caption{IHT is effective only when its smoothing direction matches the outlier structure.
We report the relative MSE reduction from IHT over all nine outlier-pattern pairs.
Results are shown for synthetic tensors with five random seeds (Left)
and real tensors from Llama3.2-3B's \texttt{block.11.feed\_forward.w2} (Right).
Improvement is defined as
$(E_{\mathrm{base}} - E_{\mathrm{IHT}})/E_{\mathrm{base}} \times 100$
; see \Cref{sec:appendix_setup} for the full setup and error definitions.
}
    \vspace{-10pt}
    \label{fig:mse_analysis}
\end{figure*}

Several consistent trends emerge. First, not all theoretical pattern pairs appear in practice: RR, NR, and CR are absent because weights are consistently classified as None, while activations are classified as either Column-wise or None. Second, the pattern distribution is strongly path-dependent. In the forward path, $Y = XW^\top$, only \textbf{CN} and \textbf{NN} pairs arise, both of which are compatible with IHT. In contrast, the $\nabla W$ path, $G_W = G_Y^\top X$, is dominated by \textbf{RC, RN, NC} and \textbf{CC} pairs, for which a single fixed transform is insufficient. The $\nabla X$ path, $G_X = G_YW$, falls between these cases, containing a mixture of \textbf{RN, NN,} and \textbf{CN} pairs. These observations show that the challenge is not merely the presence of outliers, but the fact that their pattern pairs differ systematically across forward and backward computation paths. Consequently, the appropriate low-precision treatment cannot be chosen globally; it must depend on the operand patterns of each matrix multiplication.

Finally, we examine whether these patterns are stable enough to be used for such a decision. Tracking Llama3.2-3B over 300 training steps shows that pattern assignments stabilize early and remain consistent throughout training (\Cref{fig:op_per_step}; additional layers in \Cref{sec:OP_figure}). This temporal stability enables a practical calibration-based design: \textit{patterns can be detected once in a short initial phase and then reused to guide strategy selection during training.} The next section builds on this observation to introduce AdaHOP.
\section{Methodology}
\label{sec:method}
Based on the analysis above, we propose AdaHOP, an adaptive outlier-pattern-aware framework that assigns each matrix multiplication a suitable low-precision strategy based on its operand pattern pair, as illustrated in \Cref{fig:pipeline}. AdaHOP detects tensor-wise outlier patterns through a short BF16 calibration phase, then uses the detected pattern pairs to select the computation strategy for each matrix multiplication. We describe the calibration and strategy selection below, followed by the theoretical justification and kernel implementation.

\begin{figure*}[t]
    \centering
    \includesvg[width=\textwidth]{figs/OP_per_step.svg}
    \caption{Outlier patterns across 300 training steps for Llama3.2-3B. Patterns remain stable throughout training, enabling one-time calibration. Other layers in \Cref{sec:OP_figure}.}
    \vspace{-8pt}
    \label{fig:op_per_step}
\end{figure*}

\begin{figure*}[t]
    \centering
    \includegraphics[width=\textwidth]{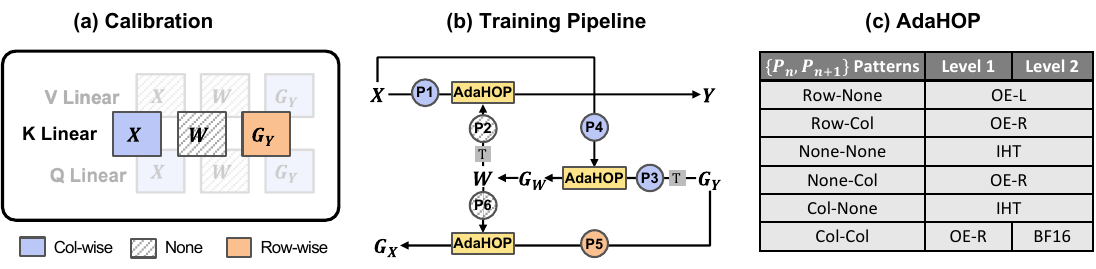}
    \caption{AdaHOP pipeline. For each linear layer's three matrix multiplications, AdaHOP selectively applies IHT, OE (Left or Right), or BF16 based on the detected outlier pattern pair $P_n$ ($n=1,3,5$).}
    \vspace{-10pt}
    \label{fig:pipeline}
\end{figure*}

\subsection{Calibration Training for Outlier Pattern Detection}
\label{sec:calibration}

Exploiting the temporal stability observed in \Cref{sec:patterns_in_layers}, AdaHOP performs a 30-step BF16 calibration at the start of training, computing the dimension-normalized CV of each tensor along both dimensions. Each tensor is classified as Row-wise, Column-wise, or None using threshold $\tau = 2.0$ (\Cref{sec:appendix_calibration}), with the final pattern determined by majority vote. This one-time calibration incurs negligible overhead (<0.01\% of the training budget) and eliminates runtime pattern detection.

\subsection{Adaptive Hadamard Strategy Selection}
\label{sec:adaptive_strategy}

Given the calibrated outlier pattern pair, AdaHOP selects a computation strategy for each matrix multiplication. \Cref{fig:pipeline} summarizes the complete strategy assignment. When IHT already smooths the dominant outlier directions, AdaHOP applies IHT alone. When IHT is misaligned with the outlier structure, AdaHOP combines IHT with Outlier Extraction (OE), which separates the dominant outlier rows or columns into a small BF16 path while applying MXFP4 IHT to the residual tensor. 
This decomposition is possible because the outliers identified in \Cref{sec:outlier_pattern_def} are structured: Row-wise outliers occupy entire rows, and Column-wise outliers occupy entire columns. Thus, the corresponding sub-matrix can be cleanly isolated and multiplied in BF16, while the remaining residual matrix is smoothed and quantized with IHT.

For pattern pairs where inner-dimension smoothing is effective, namely
\textbf{CN} and \textbf{NN} (and \textbf{CR}, although it does not appear in practice), AdaHOP uses IHT alone. For pairs where IHT is ineffective, AdaHOP extracts the dominant outlier component before applying IHT to the residual. For \textbf{RN}, the outliers lie in the rows of $A$, so AdaHOP applies \textbf{OE-Left}. We decompose $A = A_{\mathrm{res}} + A_{\mathrm{out}}$, where $A_{\mathrm{out}}$ contains the top-$r$ outlier rows ($r = 64$; see \Cref{sec:appendix_block_size}), and compute
\begin{equation}
    C =
    \underbrace{Q(A_{\mathrm{res}} H_k) \cdot Q(H_k^\top B)}_{\mathrm{MXFP4}}
    +
    \underbrace{A_{\mathrm{out}} \cdot B}_{\mathrm{BF16}} .
    \label{eq:oe_left}
\end{equation}
For \textbf{RC}, \textbf{NC}, and \textbf{CC}, AdaHOP applies \textbf{OE-Right}: it decomposes $B = B_{\mathrm{res}} + B_{\mathrm{out}}$, where $B_{\mathrm{out}}$
contains the top-$r$ outlier columns, and computes
\begin{equation}
    C =
    \underbrace{Q(A H_k) \cdot Q(H_k^\top B_{\mathrm{res}})}_{\mathrm{MXFP4}}
    +
    \underbrace{A \cdot B_{\mathrm{out}}}_{\mathrm{BF16}} .
    \label{eq:oe_right}
\end{equation}

This right-side extraction is particularly important for the \textbf{RC}, \textbf{NC}, and \textbf{CC} pairs, which predominantly arise in the $\nabla W$ path with $A = G_Y^\top$ and $B = X$. In these cases, extracting the activation-side column outliers yields the largest error reduction under a fixed extraction budget, since activations exhibit higher kurtosis than gradients (\Cref{sec:appendix_setup}).

The \textbf{CC} pair is rare, accounting for less than 6\% of all paths, but it appears in attention-critical Key/Value projections~\cite{kivi, gear, xkv}. AdaHOP therefore provides two quality--efficiency levels. \textbf{AdaHOP-Lv1} uses OE-Right with IHT on the residual, preserving most computation in MXFP4. \textbf{AdaHOP-Lv2} computes CC pairs fully in BF16, prioritizing training quality for these sensitive paths.

\subsection{Accuracy--Efficiency Design of AdaHOP}
\label{sec:accuracy_efficiency_design}

AdaHOP preserves IHT as the default low-cost primitive while correcting its pattern-dependent failure cases. However, for misaligned pairs such as \textbf{RC}, IHT leaves the dominant outlier unsmoothed. OHT can recover the complementary smoothing and improves the error bound  by $O(\sqrt{mn})$ (\Cref{thm:rowcol_error_appendix}), but requires additional transforms and inverse rotations, incurring substantial data movement.

AdaHOP instead uses OE as a targeted correction. Since dominant outliers appear as rows or columns, OE isolates only the corresponding outlier submatrix into a BF16 path and applies MXFP4 IHT to the residual. This removes the dominant outlier amplification term from the error bound. For row-wise outliers in $A$, \Cref{thm:outlier_extract_appendix} reduces the bound from
\[
O\!\left(
\epsilon_{\mathrm{quant}} \|A\|_F \|B\|_F
\sqrt{\gamma(A)\gamma(B)}
\right)
\]
to
\[
O\!\left(
\epsilon_{\mathrm{quant}} \|A_{\mathrm{res}}\|_F \|B\|_F
\sqrt{\gamma(H_k^\top B)}
\right),
\]
where $A_{\mathrm{res}}$ is the residual after extracting the dominant outlier rows. Thus, AdaHOP achieves a pattern-aware accuracy--efficiency trade-off: most arithmetic remains in the MXFP4 path, while only a small sensitive component is computed in BF16. Complete proofs are provided in \Cref{sec:appendix_theory}.

\subsection{Hardware-Aware Implementation}
\label{sec:hardware_implementation}

AdaHOP makes OE practical by aligning the algorithm with AMD CDNA4's mixed-precision execution capability. Because OE extracts only a small structured outlier submatrix, the BF16 outlier path can be executed together with the dominant MXFP4 residual path, rather than as a separate full-size computation. Thus, the numerical benefit of OE is obtained with only a small marginal cost. We implement this design as fused Triton kernels optimized for AMD CDNA4. The pipeline consists of Fast Outlier Index Detection (FOID), Fused IHT + MXFP4 quantization, mixed-precision GEMM, and scatter-add kernel, minimizing intermediate tensors and data movement.
We report the detailed implementation and its latency breakdown in \Cref{sec:appendix_hardware} and \Cref{tab:latency}, respectively.

\begin{table*}[t]
\centering
\caption{Zero-shot accuracy (\%) after training on C4. Best MXFP4-based result is in \textbf{bold}; second best is \underline{underlined}.}
\label{tab:downstream}
\small
\setlength{\tabcolsep}{3.2pt}
\renewcommand{\arraystretch}{0.92}
\begin{tabular*}{\textwidth}{@{\extracolsep{\fill}}l c l c c c c c@{}}
\toprule
\textbf{Model} & \textbf{Tokens} & \textbf{Method} &
\textbf{PIQA} & \textbf{HellaSwag} & \textbf{ARC-E} & \textbf{Lambada} & \textbf{Avg.} \\
\midrule
\multirow{7}{*}{Llama3.2-1B} & \multirow{7}{*}{40B}
 & BF16 & 72.69 & 52.33 & 47.85 & 38.17 & 52.76 \\
 & & MXFP4 & 64.69 & 37.94 & 37.29 & 11.93 & 37.96 \\
 & & MXFP4+Had. & 71.85 & 51.37 & 45.91 & 37.68 & 51.70 \\
 & & Tseng et al. & 72.14 & 50.95 & 46.17 & 36.81 & 51.52 \\
 & & HALO & 72.09 & 51.89 & 46.72 & \textbf{39.03} & 52.43 \\
 & & AdaHOP-Lv1 & \textbf{73.04} & \underline{52.24} & \underline{46.91} & 38.21 & \underline{52.60} \\
 & & AdaHOP-Lv2 & \underline{73.01} & \textbf{52.32} & \textbf{47.05} & \underline{38.55} & \textbf{52.73} \\
\midrule
\multirow{7}{*}{Instella-3B} & \multirow{7}{*}{40B}
 & BF16 & 73.52 & 54.31 & 48.91 & 46.26 & 55.75 \\
 & & MXFP4 & 68.17 & 42.07 & 40.91 & 27.71 & 44.72 \\
 & & MXFP4+Had. & 72.69 & 53.54 & 48.61 & 43.39 & 54.56 \\
 & & Tseng et al. & \underline{73.45} & 53.19 & 48.86 & 45.12 & 55.16 \\
 & & HALO & 72.20 & \underline{53.54} & 47.81 & \textbf{45.90} & 54.86 \\
 & & AdaHOP-Lv1 & 73.23 & 53.32 & \underline{49.41} & 45.02 & \underline{55.25} \\
 & & AdaHOP-Lv2 & \textbf{73.83} & \textbf{54.26} & \textbf{50.67} & \underline{45.42} & \textbf{56.05} \\
\midrule
\multirow{5}{*}{Llama3.2-3B} & \multirow{5}{*}{60B}
 & BF16 & 73.72 & 57.75 & 48.53 & 47.53 & 56.88 \\
 & & MXFP4 & 55.88 & 26.30 & 29.25 & 6.09 & 29.38 \\
 & & HALO & \underline{73.45} & \textbf{57.34} & 47.98 & \textbf{47.95} & 56.68 \\
 & & AdaHOP-Lv1 & \textbf{73.67} & \underline{57.18} & \underline{48.44} & \underline{47.41} & \underline{56.68} \\
 & & AdaHOP-Lv2 & \textbf{73.67} & 57.14 & \textbf{49.12} & 47.28 & \textbf{56.80} \\
\midrule
\multirow{5}{*}{Qwen3-4B} & \multirow{5}{*}{80B}
 & BF16 & 75.30 & 62.77 & 54.63 & 52.05 & 61.19 \\
 & & Tseng et al. & \underline{75.30} & 61.90 & \textbf{54.67} & 50.26 & 60.53 \\
 & & HALO & 74.46 & 62.22 & 52.57 & \underline{51.68} & 60.23 \\
 & & AdaHOP-Lv1 & \textbf{75.68} & \underline{62.36} & \underline{53.75} & \textbf{52.00} & \textbf{60.95} \\
 & & AdaHOP-Lv2 & 74.68 & \textbf{63.38} & 53.70 & 50.94 & \underline{60.68} \\
\midrule
\multirow{4}{*}{Llama3.1-8B} & \multirow{4}{*}{160B}
 & BF16 & 75.79 & 64.19 & 57.11 & 49.64 & 61.68 \\
 & & HALO & 74.98 & 63.17 & 55.31 & 48.04 & 60.37 \\
 & & AdaHOP-Lv1 & \underline{75.43} & \underline{64.94} & \textbf{56.31} & \textbf{49.84} & \textbf{61.63} \\
 & & AdaHOP-Lv2 & \textbf{76.28} & \textbf{64.98} & \underline{55.89} & \underline{48.57} & \underline{61.43} \\
\bottomrule
\end{tabular*}
\vspace{-10pt}
\end{table*}
\section{Experimental Results}
\label{sec:experiments}

We implement AdaHOP on Torchtitan~\cite{torchtitan} and evaluate it on Llama3.2-1B, Llama3.2-3B, Llama3.1-8B~\cite{Llama3}, Instella-3B~\cite{Instella}, and Qwen3-4B~\cite{Qwen3}, trained on C4~\cite{C4}. We compare against BF16, naive MXFP4, MXFP4+Hadamard, Tseng et al.~\cite{Albert}, and HALO~\cite{HALO}. We report training loss, zero-shot accuracy on PIQA~\cite{PIQA}, HellaSwag~\cite{HellaSwag}, ARC-Easy~\cite{ARC}, and LAMBADA~\cite{LAMBADA}, as well as memory and throughput. Detailed setups are provided in \Cref{sec:appendix_setup}.

\subsection{Q1. Does pattern-aware selection reduce the BF16 loss gap?}
\label{sec:main_results}

\Cref{fig:loss_analysis} compares training dynamics including loss and the loss gap relative to BF16 for models trained from scratch. AdaHOP consistently yields the smallest BF16 gap among MXFP4-based methods, keeping the loss within 0.01 of BF16 on both Llama3.2-1B and Instella-3B. In contrast, naive MXFP4 exhibits severe degradation, while prior Hadamard-based methods only partially close the gap because fixed transform assignments do not account for diverse outlier patterns across computation paths. This demonstrates that AdaHOP's pattern-aware selection effectively suppresses quantization error during training, enabling MXFP4 training to closely match BF16-level convergence.

\subsection{Q2. Does the loss improvement translate to downstream accuracy?}
\label{sec:downstream}

Having established that AdaHOP closely follows BF16 during scratch training, we next evaluate whether this improvement translates to downstream performance. \Cref{tab:downstream} reports zero-shot accuracy on four benchmarks after training. Across models from Llama3.2-1B to Llama3.1-8B, AdaHOP achieves the best average accuracy among MXFP4-based methods, remaining within 0.24 points of BF16 on average. On Instella-3B, AdaHOP-Lv2 even surpasses BF16 by 0.30 points, and on Llama3.1-8B, AdaHOP-Lv1 outperforms HALO by 1.26 points, nearly closing the BF16 gap. These results indicate that AdaHOP's training-quality gains translate into preserved zero-shot capability, and that its pattern-aware strategy remains effective as model scale increases.

\begin{table*}[t]
\centering
\caption{(Left) End-to-end efficiency on Llama3.1-8B training with \texttt{seq\_len=4096} and \texttt{BS=8}. Memory (GB), Speed (Token/sec), and acceleration ($\times$) are measured for the full training pipeline. Quality: \textcolor[RGB]{0,150,0}{\checkmark} within 0.3pt of BF16, \textcolor[RGB]{230,150,0}{\ding{115}} within 1pt, \textcolor{red}{\ding{55}} otherwise, based on the average downstream score across all models. (Right) kernel latency (ms) breakdown for two representative Llama3.1-8B GEMMs.}
\label{tab:efficiency_latency}
\small
\begin{minipage}[t]{0.52\textwidth}
\centering
\label{tab:efficiency}
\setlength{\tabcolsep}{6pt}
\begin{tabular}{@{}lcccc@{}}
\toprule
\textbf{Method} & \textbf{Mem.} & \textbf{Speed} & \textbf{Accel.} & \textbf{Quality} \\
\midrule
BF16 & 76.00 & 10930 & 1.00 & \textcolor[RGB]{0,150,0}{\checkmark} \\
MXFP4 & 20.19 & 17640 & 1.61 & \textcolor{red}{\ding{55}} \\
MXFP4+Had. & 20.60 & 16305 & 1.49 & \textcolor{red}{\ding{55}} \\
Tseng et al. & 20.60 & 14881 & 1.36 & \textcolor[RGB]{230,150,0}{\ding{115}} \\
HALO & 20.60 & 14251 & 1.30 & \textcolor[RGB]{230,150,0}{\ding{115}} \\
\textbf{AdaHOP-Lv1} & \textbf{20.94} & 15804 & \textbf{1.45} & \textcolor[RGB]{0,150,0}{\checkmark} \\
\textbf{AdaHOP-Lv2} & \textbf{28.04} & 15922 & \textbf{1.46} & \textcolor[RGB]{0,150,0}{\checkmark} \\
\bottomrule
\end{tabular}
\end{minipage}%
\hfill
\begin{minipage}[t]{0.46\textwidth}
\centering
\label{tab:latency}
\setlength{\tabcolsep}{5pt}
\begin{tabular}{@{}l cc@{}}
\toprule
\textbf{Module} & (1) & (2) \\
\midrule
BF16 GEMM & 4.76 & 1.05 \\
\midrule
FOID & 0.03 & 0.04 \\
IHT and Quant & 0.45 & 0.13 \\
Fused GEMM & 1.65 & 0.64 \\
Fused Scatter-Add & 0.07 & 0.02 \\
\midrule
\textbf{AdaHOP Total} & \textbf{2.22 (2.13$\times$)} & \textbf{0.84 (1.26$\times$)} \\
\bottomrule
\end{tabular}
\end{minipage}
\vspace{-8pt}
\end{table*}

\begin{figure*}[t]
    \centering
    \includesvg[width=\textwidth]{figs/latency_per_batch.svg}
    \caption{Throughput per batch size on Llama3.1-8B with \texttt{seq\_len=4096}. Highlighted number is the speedup of AdaHOP-Lv2. As the batch size increases, acceleration is maximized.}
    \vspace{-8pt}
    \label{fig:latency_per_batch}
\end{figure*}

\subsection{Q3. Does OE preserve the efficiency benefits of MXFP4 training?}
\label{sec:efficiency}
\Cref{fig:latency_per_batch} reports end-to-end training throughput per batch size on Llama3.1-8B. As the batch size increases, AdaHOP's acceleration is maximized and achieves 1.46$\times$ at BS=8. Detailed memory, speed, and acceleration for BS=8 are reported in \Cref{tab:efficiency} (Left). AdaHOP-Lv1 achieves BF16-level quality with 3.6$\times$ memory compression and 1.45$\times$ acceleration, while AdaHOP-Lv2 provides a higher-fidelity option with 2.7$\times$ compression and 1.46$\times$ acceleration. Unlike MXFP4 or MXFP4+Hadamard, AdaHOP preserves quality while outperforming Tseng et al.\ and HALO in throughput.

\Cref{tab:latency} (Right) further isolates kernel behavior on two representative Llama3.1-8B GEMMs: (1) MLP projection (m, n, k) = 14336 $\times$ 4096 $\times$ 32768 and (2) Attention projection 4096 $\times$ 4096 $\times$ 32768. AdaHOP achieves 2.13$\times$ and 1.26$\times$ speedup over BF16 for MLP and attention projections, respectively. FOID and scatter-add are negligible, and the BF16 outlier path adds limited overhead, validating the OE+IHT design. The smaller end-to-end speedup reflects non-GEMM operations in the full pipeline. These measurements confirm that AdaHOP preserves most of the MXFP4 speed advantage while avoiding the transform overhead of global outer-dimension rotations.

\section{Conclusion}
\label{sec:conclusion}
We presented AdaHOP, a pattern-aware framework for accurate and efficient MXFP4 training. Our analysis shows that outliers in LLM training tensors form structured and stable Row-wise, Column-wise, and None patterns, and that the effectiveness of Hadamard smoothing depends on the outlier pattern pair rather than the computation path alone. Based on this insight, AdaHOP uses IHT as an efficient default and selectively applies OE when dominant structured outliers must be handled separately. With lightweight calibration and hardware-aware Triton kernels on AMD CDNA4, AdaHOP achieves BF16-level training quality at MXFP4 precision, with up to 3.6$\times$ memory compression and 2.13$\times$ kernel-level speedup.

\section{Limitations and Future Work}
\label{sec:limitations}

AdaHOP currently uses fixed Walsh-Hadamard matrices; combining AdaHOP with learned rotations~\cite{SpinQuant} could yield further improvements. The pattern analysis spans Llama-family and Instella models (1B to 8B); extending to architectures with different attention or normalization schemes (e.g., Mixtral~\cite{mixtral}, Gemma~\cite{gemma}) is a promising direction. The framework currently targets MXFP4, but the pattern-aware design is format-agnostic. Finally, further ablations on the extraction count $r$ and adaptive per-layer selection could improve the accuracy–efficiency trade-off.

\bibliographystyle{plainnat}
\bibliography{custom}


\newpage
\appendix

\section{Calibration Algorithm Details}
\label{sec:appendix_calibration}

\paragraph{Outlier Pattern Detection via CV.}
For a given tensor $T \in \mathbb{R}^{m \times n}$, we first compute the per-row and per-column variance vectors:
\begin{align}
    v^{\text{row}}_i &= \text{Var}(T_{i,:}), \quad i = 1, \ldots, m, \\
    v^{\text{col}}_j &= \text{Var}(T_{:,j}), \quad j = 1, \ldots, n.
\end{align}
We then measure how non-uniformly these variances are distributed by computing the coefficient of variation (CV) of each variance vector:
\begin{align}
    \text{CV}_{\text{row}} &= \frac{\text{std}(v^{\text{row}})}{\text{mean}(v^{\text{row}}) + \epsilon}, \\
    \text{CV}_{\text{col}} &= \frac{\text{std}(v^{\text{col}})}{\text{mean}(v^{\text{col}}) + \epsilon},
\end{align}
where $\epsilon$ is a small constant for numerical stability. A high $\text{CV}_{\text{row}}$ indicates that a few rows have much larger variance than others (row-wise outliers), while a high $\text{CV}_{\text{col}}$ indicates the same for columns.

To ensure size-independent comparison, we normalize each CV by its expected value under an IID assumption. For IID data, sample variance follows a scaled chi-squared distribution with $\text{CV}(\text{Var}) \approx \sqrt{2/(n_s - 1)}$, where $n_s$ is the number of samples per variance estimate. Since row variances are computed over $n$ columns and column variances over $m$ rows, we normalize:
\begin{align}
    \hat{\text{CV}}_{\text{row}} &= \text{CV}_{\text{row}} \;/\; \sqrt{2/(n - 1)}, \\
    \hat{\text{CV}}_{\text{col}} &= \text{CV}_{\text{col}} \;/\; \sqrt{2/(m - 1)}.
\end{align}
For IID data, $\hat{\text{CV}}_{\text{row}} \approx \hat{\text{CV}}_{\text{col}} \approx 1$, so any deviation from unity signals structured outliers. This normalization corrects for the statistical artifact where non-square matrices would exhibit spurious patterns due to different sample counts for row vs.\ column variance estimation. The outlier pattern is classified by comparing the \emph{ratio} of the two normalized CVs against a threshold $\tau$:
\begin{itemize}
    \item \textbf{Row-wise}: if $\hat{\text{CV}}_{\text{row}} / \hat{\text{CV}}_{\text{col}} > \tau$ (row variances are disproportionately non-uniform, indicating outlier rows).
    \item \textbf{Column-wise}: if $\hat{\text{CV}}_{\text{col}} / \hat{\text{CV}}_{\text{row}} > \tau$ (column variances are disproportionately non-uniform, indicating outlier columns).
    \item \textbf{None}: if neither ratio exceeds $\tau$.
\end{itemize}
In our experiments, we use $\tau = 2.0$.

\section{Why Variance Is Used for Outlier Detection}
\label{sec:appendix_why_variance}

In AdaHOP, the calibration (\Cref{sec:appendix_calibration}) and Fast Outlier Index Detection (\Cref{sec:appendix_hardware}) stages identify outlier rows or columns by ranking candidates according to their within-row or within-column \emph{variance}. This section provides empirical justification for this design choice over the more intuitive alternative of using mean magnitude.

The key observation is that outlier rows or columns in LLM training tensors do not consist entirely of uniformly large values. Instead, they exhibit a mixture of a few very large values and many ordinary-magnitude values within the same row or column. This heterogeneous composition is precisely what makes these features problematic for block-scaled quantization: the large values dominate the scaling factor, causing severe rounding error for the smaller values in the same quantization group. Because variance directly measures this within-feature magnitude dispersion, it is a natural proxy for quantization difficulty.

To validate this reasoning, we measure three per-column statistics for the activation tensor of Llama3.2-3B's \texttt{block.11.feed\_forward.w2}: (i)~mean magnitude, (ii)~variance, and (iii)~quantization error under MXFP4 quantization. \Cref{fig:mean_var_quant} visualizes these three metrics across all columns. In each subplot, the top-32 columns according to the respective metric are highlighted in red, while the remaining columns are shown in light blue. The figure reveals a strong visual correspondence: columns with high quantization error largely overlap with those exhibiting high variance, and to a slightly lesser extent, with those having high mean magnitude.

\begin{figure*}[t]
    \centering
    \includesvg[width=\textwidth]{figs/activation_mean_var_qerror.svg}
    \caption{Per-column mean magnitude, variance, and quantization error of the activation tensor from Llama3.2-3B's \texttt{block.11.feed\_forward.w2}. Red bars indicate the top-32 columns for each metric; light blue bars indicate the remaining columns.}
    \label{fig:mean_var_quant}
\end{figure*}

To quantify this correspondence, we compute the Spearman rank correlation between each detection metric and the per-column quantization error. Variance achieves a Spearman correlation of $\rho = 0.9392$ with quantization error, compared to $\rho = 0.8720$ for mean magnitude. The higher correlation confirms that variance is a more faithful predictor of which columns will incur the largest quantization error, and a more effective criterion for selecting outlier features to extract into the high-precision path.

Intuitively, a column with large but uniform values has high mean magnitude yet low variance; such a column is well-served by block-scaled quantization because a single scale factor accurately represents all values. Conversely, a column containing a mixture of extreme and moderate values has high variance and is poorly represented by any single scale, leading to large quantization error. Variance captures exactly this distinction, which mean magnitude does not.

These results provide a principled justification for using variance as the outlier detection criterion in FOID: it most closely tracks the quantity we ultimately aim to minimize quantization error while remaining computationally inexpensive enough for online detection during training.

\section{Outlier Patterns of Llama3.2-3B}
\label{sec:OP_figure}

\begin{figure*}[t]
    \centering
    \includesvg[width=\textwidth]{figs/OP_per_layers.svg}
    \caption{Outlier patterns of weight ($W$), activation ($X$), and gradient ($G_Y$) tensors across 300 training steps for all 12 representative blocks of Llama3.2-3B. Each row represents a tensor from a specific layer, and the color indicates the detected outlier pattern at each step. This extended view reveals depth-dependent transitions in gradient outlier patterns: Row-wise patterns appear in early blocks, fade to None in the middle layers, and re-emerge selectively in the K-projection gradient toward the final blocks.}
    \vspace{-8pt}
    \label{fig:op_per_layers}
\end{figure*}

 
\Cref{fig:op_per_layers} visualizes the outlier patterns of all 12 representative blocks of Llama3.2-3B across 300 training steps. Three distinct phases emerge along the depth dimension.
 
In early blocks (e.g., \texttt{block.0}, \texttt{block.1}), the gradients of K, V, and O projections exhibit Row-wise outliers. Near the embedding layer, each weight row maps directly to a token-level representation. The loss signal therefore concentrates on a few high-impact tokens, producing token-wise (row-wise) outlier structure in the gradients.
 
In middle and later blocks, which constitute the majority of the network, gradient patterns are predominantly Column-wise or None. In middle blocks (\texttt{block.5} to \texttt{block.21}), repeated attention, feed-forward, and layer normalization operations diffuse token-specific gradient energy across feature dimensions. Many gradient tensors in these layers exhibit Column-wise patterns, consistent with the dominance of CN pairs in the $\nabla X$ path shown in \Cref{tab:pattern_distribution}.
 
In late blocks (\texttt{block.22} to \texttt{block.27}), Row-wise outliers re-emerge selectively in the K-projection gradient. Near the output, attention tends to concentrate on a small set of anchor tokens~\cite{AttentionSink}. This creates an asymmetry in the attention gradient. The key gradient $\nabla_K S = \nabla_S^\top Q / \sqrt{d}$ aggregates contributions from all query positions onto each key token. When many queries attend to the same anchor, gradient signals accumulate on that token's row, producing a row-wise outlier. In contrast, $\nabla_Q S = \nabla_S K / \sqrt{d}$ distributes gradients independently across query positions, so no single row dominates. This asymmetry explains why the row-wise pattern recurs in K but not in Q. Why this transition occurs specifically around block~22 is left for future investigation.
 
Across all three phases, weight and activation patterns remain stable throughout training, consistent with \Cref{sec:patterns_in_layers}. The depth-dependent variation of gradient patterns is structurally rich, yet temporally stable, making them reliably detectable via a short calibration phase.

\section{Detailed Experimental Settings}
\label{sec:appendix_setup}

\textbf{Hardware Environment.}
All experiments are conducted on AMD Instinct™ MI350X GPUs with the ROCm 7.0.2 software stack. Custom fused kernels are implemented using Triton 3.6, targeting the AMD CDNA4 architecture.

\textbf{\Cref{sec:why_patterns_matter}.} For the real tensor analysis, we extract the gradient, activation, and weight tensors from Llama3.2-3B's \texttt{block.11.feed\_forward.w2} and trim them to the first 1024 $\times$ 1024 tile. The resulting kurtosis values are 16.00 for the gradient (originally Column-wise tensor, transposed to simulate Row-wise pattern), 222.00 for the activation (Column-wise), and 2.88 for the weight (None).

In the synthetic tensor analysis, we use 4096 $\times$ 4096 synthetic tensors A and B with 5 random seeds (42, 256, 1925, 2048, 7777) and amplify random columns' magnitudes by 25.00 for Column-wise tensors and 5.00 for Row-wise tensors. This is because Column-wise tensors tend to correspond to activations (see \Cref{tab:pattern_distribution}) and exhibit higher kurtosis across LLMs. The resulting mean kurtosis values are Row: 9.08, Column: 208.30 when outliers are present, and None: 3.0. 

For MSE improvement calculation, the error definition is $E_{\mathrm{base}}=\mathrm{MSE}(Q(A)Q(B),AB)$ and $E_{\mathrm{IHT}}=\mathrm{MSE}(Q(AH_k)Q(H_k^\top B),AB)$.

\textbf{\Cref{sec:experiments}.} We evaluate AdaHOP on five LLM architectures of varying scales: Llama3.2-1B, Llama3.2-3B, Instella-3B, Qwen3-4B and Llama3.1-8B. 
All models are trained on the C4 (Colossal Clean Crawled Corpus) dataset \cite{C4} with a sequence length of 4{,}096 tokens, local batch size of 8 and global batch size of 128. 
We compare AdaHOP against five baseline methods: BF16 full-precision training, which serves as an upper bound on quality; naive MXFP4 quantization without any outlier suppression; MXFP4+Hadamard, which applies uniform IHT to all layers; Tseng et al. \cite{Albert}, a recent MXFP4 training method; and HALO \cite{HALO}, which applies OHT to gradient computation paths. 
We report training loss, zero-shot accuracy on downstream benchmarks (PIQA \cite{PIQA}, HellaSwag \cite{HellaSwag}, ARC-Easy \cite{ARC}, LAMBADA \cite{LAMBADA}), memory consumption, and training throughput.

All models are trained using the AdamW optimizer~\cite{AdamW} with a learning rate of $4 \times 10^{-4}$ and an epsilon of $10^{-8}$, following a linear warmup schedule over 200 steps.
Following the Chinchilla optimal scaling law~\cite{Chinchilla}, the number of training tokens is scaled with model size: Llama3.2-1B is trained on 40B tokens, Llama3.2-3B on 60B tokens, Qwen3-4B on 80B tokens, and Llama3.1-8B on 160B tokens. For Instella-3B, training tokens are 40B due to a time limitation. 
Gradient norms are clipped to a maximum of 1.0 throughout training.

\section{Theoretical Analysis}
\label{sec:appendix_theory}

This section presents the key theoretical results and complete proofs that justify AdaHOP's transform selection strategy. 
While IHT and OHT are mathematically equivalent in infinite precision, their quantization error characteristics differ significantly depending on the outlier structure. 
The central insight is that the Hadamard transform mixes values along the dimension it is applied to: right multiplication $XH_n$ mixes values within each row (across columns), while left multiplication $H_mX$ mixes values within each column (across rows). 
Consequently, a transform is effective only when it is applied orthogonally to the outlier direction. 
The following proposition formalizes this, and the complete pattern-transform correspondence is summarized in the table below.

\begin{proposition}[Transform Effectiveness for Outlier Patterns]\label{prop:transform_effectiveness}
For matrix $A \in \mathbb{R}^{m \times k}$ with row outliers, left multiplication reduces the outlier factor by $\gamma(H_m A) \approx \frac{1}{m}\gamma(A)$, while right multiplication leaves it unchanged: $\gamma(A H_k) \approx \gamma(A)$. Symmetrically, for matrix $B \in \mathbb{R}^{k \times n}$ with column outliers, right multiplication yields $\gamma(B H_n) \approx \frac{1}{n}\gamma(B)$, while left multiplication is ineffective.
\end{proposition}

\begin{center}
\footnotesize
\begin{tabular}{|c|c|c|c|}
\hline
\textbf{A} & \textbf{B} & \textbf{Optimal Transform} & \textbf{Improvement}$^*$ \\
\hline
Row & Col & OHT & $O(\sqrt{mn})$ \\
Row & Row & Partial$^\dagger$ & $O(\sqrt{k})$ (IHT) / $O(\sqrt{m})$ (OHT) \\
Col & Col & Partial$^\dagger$ & $O(\sqrt{k})$ (IHT) / $O(\sqrt{n})$ (OHT) \\
Col & Row & IHT & $O(k)$ \\
\hline
\end{tabular}
\end{center}
{\footnotesize $^*$Improvement factor over no transform, where $k$ denotes the shared (contraction) dimension of $C = AB$.\\
$^\dagger$For Row-Row, IHT smooths only $B$'s rows ($O(\sqrt{k})$) while OHT smooths only $A$'s rows ($O(\sqrt{m})$); neither fully resolves both operands. The symmetric argument holds for Col-Col. This motivates OE (\Cref{sec:outlier_extract_appendix}).}

Let $H_d \in \mathbb{R}^{d \times d}$ denote the normalized Hadamard matrix of size $d \times d$, satisfying $H_d^\top H_d = I_d$.

Recall the outlier factor defined as:
\begin{equation}
    \gamma(A) = \frac{mk \cdot \max_{i,j} |a_{ij}|^2}{\|A\|_F^2}
\end{equation}
which measures how concentrated the energy is in extreme values. We further classify outlier patterns:
\begin{itemize}
    \item \textbf{Row pattern}: Outliers are concentrated in specific rows, i.e., $\exists$ row indices $\mathcal{I}$ such that $\max_{j} |a_{ij}| \gg \text{median}_{v,j} |a_{vj}|$ for $i \in \mathcal{I}$.
    \item \textbf{Column pattern}: Outliers are concentrated in specific columns, i.e., $\exists$ column indices $\mathcal{J}$ such that $\max_{i} |a_{ij}| \gg \text{median}_{i,v} |a_{iv}|$ for $j \in \mathcal{J}$.
\end{itemize}

In low-precision matrix multiplication $C = AB$ where $A \in \mathbb{R}^{m \times k}$ and $B \in \mathbb{R}^{k \times n}$, we define the row-column pattern as the case where $A$ exhibits row outliers and $B$ exhibits column outliers. This pattern commonly arises in transformer computations where gradient tensors have token-wise outliers (row pattern) and activation tensors have channel-wise outliers (column pattern).

For computing the matrix product $C = AB$ with quantization, we consider two Hadamard transform strategies:

\textbf{IHT:}
Applied to the inner (reduction) dimension $k$. Requires $H_k$.
\begin{equation}
    C_{\text{inner}} = Q(A H_k) \cdot Q(H_k^\top B)
\end{equation}

\textbf{OHT:}
Applied to the outer dimensions $m$ and $n$. Requires $H_m$ and $H_n$.
\begin{equation}
    C_{\text{outer}} = H_m^\top \left( Q(H_m A) \cdot Q(B H_n) \right) H_n^\top
\end{equation}

\begin{lemma}[Hadamard Mixing Direction]
\label{lem:mixing_direction}
Let $H_n \in \mathbb{R}^{n \times n}$ and $H_m \in \mathbb{R}^{m \times m}$ be the normalized Hadamard matrix. For a matrix $X \in \mathbb{R}^{m \times n}$:
\begin{itemize}
    \item Right multiplication $X H_n$: mixes values \textbf{within each row} (across columns).
    \item Left multiplication $H_m X$: mixes values \textbf{within each column} (across rows).
\end{itemize}
\end{lemma}

\begin{proof}[Proof of \Cref{prop:transform_effectiveness}]
For row outliers in $A$, the high-magnitude values are concentrated in specific rows. 
Left multiplication $H_m A$ computes $(H_m A)_{ij} = \sum_p (H_m)_{ip} A_{pj}$, which mixes different rows' values in each column. If row $r$ has outliers with magnitude $\mu$, the transformed elements distribute this magnitude across all $m$ rows. Specifically, since $|(H_m)_{ip}| = 1/\sqrt{m}$, the peak magnitude is reduced by factor $\sqrt{m}$. Thus, $\|H_m A\|_{\max} \approx \frac{1}{\sqrt{m}}\|A\|_{\max}$, yielding $\gamma(H_m A) \propto (\|H_m A\|_{\max})^2 \approx \frac{1}{m}\gamma(A)$. Right multiplication $A H_k$ mixes columns within each row. Row outliers remain in their original rows since no mixing across rows occurs, hence $\gamma(A H_k) \approx \gamma(A)$. The argument for column outliers in $B$ is symmetric.
\end{proof}

For the row-column pattern ($A$ has row outliers, $B$ has column outliers), we derive the quantization error bounds.

\begin{theorem}[Error Bound for Row-Column Pattern]
\label{thm:rowcol_error_appendix}
Let $A$ have row outliers with factor $\gamma(A)$, and $B$ have column outliers with factor $\gamma(B)$. The quantization error for computing $C = AB$ satisfies:

\textbf{IHT:}
\begin{equation}
\begin{split}
    \|AB - C_{\text{inner}}\|_F \leq O\Big(\epsilon_{\text{quant}} \cdot \|A\|_F \|B\|_F \cdot \sqrt{\gamma(A) \cdot \gamma(B)}\Big)
\end{split}
\end{equation}

\textbf{OHT:}
\begin{equation}
\begin{split}
    \|AB - C_{\text{outer}}\|_F \leq O\Big(\epsilon_{\text{quant}} \cdot \|A\|_F \|B\|_F \cdot \sqrt{\tfrac{\gamma(A)}{m} \cdot \tfrac{\gamma(B)}{n}}\Big)
\end{split}
\end{equation}
\end{theorem}

\begin{proof}
For IHT, we apply:
\begin{itemize}
    \item $A H_k$ to $A$: Right multiplication does not reduce row outliers $\implies \gamma \approx \gamma(A)$.
    \item $H_k^\top B$ to $B$: Left multiplication does not reduce column outliers $\implies \gamma \approx \gamma(B)$.
\end{itemize}
The quantization error scales with the product of the dynamic ranges (determined by $\gamma$).

For OHT, we apply:
\begin{itemize}
    \item $H_m A$ to $A$: Left multiplication reduces row outliers $\implies \gamma \approx \gamma(A)/m$.
    \item $B H_n$ to $B$: Right multiplication reduces column outliers $\implies \gamma \approx \gamma(B)/n$.
\end{itemize}
The orthogonal transforms $H_m^\top$ and $H_n$ preserve the Frobenius norm of the error, so the reduction in $\gamma$ directly translates to reduced error bounds.
\end{proof}

\begin{corollary}[OHT Improvement Factor]
For the row-column pattern, OHT achieves an improvement factor of:
\begin{equation}
    \frac{\text{Error}_{\text{inner}}}{\text{Error}_{\text{outer}}} = O\left(\sqrt{mn}\right)
\end{equation}
For typical matrix dimensions in transformers ($m, n \sim 10^3$ to $10^4$), this represents a 3 to 4 order of magnitude reduction in quantization error.
\end{corollary}

\subsection{OE for Enhanced Precision}\label{sec:outlier_extract_appendix}

When outliers are severely concentrated, we introduce OE. Following the methodology in \Cref{sec:adaptive_strategy}, OE is combined with IHT.

\paragraph{OE Formulation.}
Given $A \in \mathbb{R}^{m \times k}$ with row outliers, we decompose $A = A_{\mathrm{res}} + A_{\mathrm{out}}$, where $A_{\mathrm{out}}$ contains only the top-$r$ outlier rows ($r \ll m$).
The product becomes $AB = A_{\mathrm{res}} B + A_{\mathrm{out}} B$.
Applying IHT to the clean part and computing the outlier part in BF16:
\begin{align}
    C_{\mathrm{res}} &= Q(A_{\mathrm{res}} H_k) \cdot Q(H_k^\top B) \label{eq:oe_clean}\\
    C_{\mathrm{out}} &= A_{\mathrm{out}} \cdot B \quad \text{(BF16, exact)} \label{eq:oe_outlier}
\end{align}
Since $A_{\mathrm{res}} H_k \cdot H_k^\top B = A_{\mathrm{res}} B$ in infinite precision, IHT preserves the exact product for the clean path. The outlier path is computed entirely in BF16, incurring no quantization error.

\begin{theorem}[OE Error Bound]
\label{thm:outlier_extract_appendix}
Let $A \in \mathbb{R}^{m \times k}$ have row outliers concentrated in $r$ rows ($r \ll m$). After extraction:
\begin{equation}
    \gamma(A_{\mathrm{res}}) \approx O(1)
\end{equation}
The combined error bound satisfies:
\begin{equation}
    \|AB - (C_{\mathrm{res}} + C_{\mathrm{out}})\|_F \leq O\!\left(\epsilon_{\text{quant}} \cdot \|A_{\mathrm{res}}\|_F \|B\|_F \cdot \sqrt{\gamma(H_k^\top B)}\right)
\end{equation}
where $\gamma(H_k^\top B)$ depends on $B$'s outlier pattern: $\gamma(H_k^\top B) \approx \gamma(B)/k$ if $B$ has row outliers (smoothed by left multiplication), and $\gamma(H_k^\top B) \approx O(1)$ if $B$ has the None pattern. In both cases relevant to OE-Left (RN pattern pairs per \Cref{fig:pipeline}), the error is substantially reduced compared to the baseline $O(\epsilon_{\text{quant}} \cdot \|A\|_F \|B\|_F \cdot \sqrt{\gamma(A) \cdot \gamma(B)})$.
\end{theorem}

\begin{proof}
Decompose $AB = A_{\mathrm{res}} B + A_{\mathrm{out}} B$.

The outlier path $C_{\mathrm{out}} = A_{\mathrm{out}} \cdot B$ is computed in BF16 (full precision), contributing zero quantization error. For the clean path, $A_{\mathrm{res}} H_k \cdot H_k^\top B = A_{\mathrm{res}} B$ holds exactly in infinite precision, so the error arises solely from quantization:
\begin{equation}
    \|A_{\mathrm{res}} B - Q(A_{\mathrm{res}} H_k)\, Q(H_k^\top B)\|_F
\end{equation}

After removing the top-$r$ outlier rows, $A_{\mathrm{res}}$ has a nearly uniform magnitude distribution with $\gamma(A_{\mathrm{res}}) \approx O(1)$. Right multiplication $A_{\mathrm{res}} H_k$ preserves this low outlier factor, since $A_{\mathrm{res}}$ has no pronounced row or column structure to be disrupted. For $B$, left multiplication $H_k^\top B$ mixes values across rows (within each column). If $B$ has row outliers, this reduces $\gamma(B)$ by a factor of $k$; if $B$ has the None pattern, $\gamma$ remains $O(1)$.

Thus the quantization error of the clean path scales as:
\begin{equation}
    O\!\left(\epsilon_{\text{quant}} \cdot \|A_{\mathrm{res}} H_k\|_F \cdot \|H_k^\top B\|_F \cdot \sqrt{\gamma(A_{\mathrm{res}} H_k) \cdot \gamma(H_k^\top B)}\right)
\end{equation}
Using the orthogonality of $H_k$ to preserve Frobenius norms ($\|A_{\mathrm{res}} H_k\|_F = \|A_{\mathrm{res}}\|_F$, $\|H_k^\top B\|_F = \|B\|_F$) and $\gamma(A_{\mathrm{res}} H_k) \approx O(1)$, this simplifies to:
\begin{equation}
    O\!\left(\epsilon_{\text{quant}} \cdot \|A_{\mathrm{res}}\|_F \cdot \|B\|_F \cdot \sqrt{\gamma(H_k^\top B)}\right)
\end{equation}

Compared to the baseline error without extraction, $O(\epsilon_{\text{quant}} \cdot \|A\|_F \|B\|_F \cdot \sqrt{\gamma(A) \cdot \gamma(B)})$, OE eliminates the $\sqrt{\gamma(A)}$ amplification factor, which is the dominant source of quantization error when $A$ has severe row outliers.
\end{proof}

\section{Hardware-Aware Implementation Details}
\label{sec:appendix_hardware}

\begin{figure*}[t]
    \centering
    \includesvg[width=\textwidth]{figs/adahop_implementation.svg}
    \caption{Hardware-aware implementation pipeline of AdaHOP. The pipeline consists of four stages: FOID, IHT and Quantization, Fused MXFP4+BF16 GEMM, and Fused Scatter-Add, all implemented as optimized Triton kernels targeting the AMD CDNA4 architecture. Note that the red block in Fast Outlier Index Detection represents the first 64 elements along each column.}
    \label{fig:implementation}
\end{figure*}

The AdaHOP pipeline is implemented as fused Triton kernels optimized for AMD CDNA4 architecture, as illustrated in \Cref{fig:implementation}. The implementation consists of four stages designed to minimize the computational overhead of OE while fully exploiting the hardware's mixed-precision capabilities.

\paragraph{Extraction Rank Selection.}
\label{sec:appendix_block_size}
We set $r = 64$ for the number of extracted outlier rows or columns. Although the outlier path is computed in BF16 rather than MXFP4, its size is constrained by the fused kernel's latency-hiding design. In stage~3, the MXFP4 residual GEMM and BF16 outlier GEMM execute concurrently on different Compute Units within a single kernel launch: BF16 tiles are evenly interleaved among the more numerous MXFP4 tiles via Bresenham-style scheduling, so that both tile types run in parallel from the first wave. The BF16 branch uses smaller output tiles than the MXFP4 branch so that each BF16 tile completes within the execution time of one MXFP4 tile. Setting $r = 64$ keeps the total number of BF16 tiles small enough to be fully hidden behind the dominant MXFP4 computation, adding near-zero wall-clock overhead. The rank ablation in \Cref{sec:appendix_rank_ablation} confirms that $r = 64$ achieves a favorable accuracy--efficiency trade-off.

\begin{enumerate}
    \item Fast Outlier Index Detection (FOID) exploits the fixed structure of outlier rows or columns by computing the variance of the first 64 elements along each row (or column) to quickly identify outlier indices. This metric is effective because outlier rows/columns exhibit high within-row/column variance due to the mixture of large outlier values and smaller non-outlier values; non-outlier rows/columns have uniformly small magnitudes and thus low variance. The top-$r$ rows (or columns) by variance are selected as outlier features, and the corresponding entries in the residual tensor are zeroed out.
    \item IHT and Quantization apply IHT to the residual tensor (with outliers removed) using 1D Hadamard Transform with block size 32, followed by MXFP4 quantization with 1D per-column scaling. For activation tensors in the Forward path, both the quantized residual and the BF16 outlier tensor are saved to the context for backpropagation.
    \item Fused MXFP4+BF16 GEMM exploits the Compute Unit (CU) architecture of AMD CDNA4 to execute the MXFP4 residual multiplication and BF16 outlier multiplication concurrently within a single kernel launch. The MXFP4 branch uses large output tiles for maximum throughput, while the BF16 branch uses smaller output tiles so that per-tile latency is balanced; BF16 tiles are evenly interleaved among MXFP4 tiles via Bresenham-style scheduling to ensure concurrent execution from the first wave.
    \item Fused Scatter-Add scatters and adds the result of the outlier matmul in-place to the residual matmul result using BF16 accumulation, avoiding the overhead of materializing intermediate results.
\end{enumerate}

FOID uses within-row/column variance to rank candidate outlier features. A potential concern is that a feature with large but \emph{constant} magnitude across all positions would have zero variance and be missed by this criterion. However, such a feature does not degrade quantization accuracy: because all its values share the same magnitude, they occupy a narrow sub-range and incur minimal rounding error under block-scaled quantization, making extraction unnecessary. In practice, outlier features in LLM training tensors consistently exhibit high magnitude \emph{variation}---a mixture of large outlier values and smaller non-outlier values within each row or column---which is precisely what variance captures. This makes variance both an efficient and effective detection metric for the structured outlier patterns observed during training.

\section{Rank Ablation Study for Outlier Extraction}
\label{sec:appendix_rank_ablation}

We conduct an ablation study on the extraction rank $r$ used in OE, measuring both the quantization error~(MSE) and kernel latency as the rank varies over $\{0, 16, 32, 64, 128, 256\}$. 
\Cref{fig:rank_ablation} summarizes the results for both real and synthetic tensors under the \textbf{RC} pattern pair, where IHT alone is ineffective and OE is most critical.

\begin{figure*}[t]
    \centering
    \includesvg[width=\textwidth]{figs/rank_ablation.svg}
    \caption{MSE and latency as a function of the OE extraction rank for (Left)~real tensors from Llama3.1-8B and (Right)~synthetic tensors. Rank~64 achieves a favorable accuracy--efficiency trade-off in both settings: MSE is substantially reduced while latency remains close to the minimum.}
    \label{fig:rank_ablation}
\end{figure*}

\paragraph{Real Tensor.}
We extract the activation and gradient tensors from Llama3.1-8B's \texttt{block.25.feed\_forward.w2} and crop both to $4096 \times 4096$. 
For forming an \textbf{RC} matrix multiplication scenario, the gradient tensor is transposed to produce a Row-wise outlier tensor, while the activation tensor is used as-is as the Column-wise outlier tensor. 
As shown in \Cref{fig:rank_ablation}~(Left), increasing the extraction rank consistently reduces MSE: at rank 64, the MSE drops to $3.46 \times 10^{-7}$ from $4.69 \times 10^{-7}$ at rank~0, while the latency increases only marginally from 0.077\,ms to 0.113\,ms. 
However, when the rank is further increased to 128, the latency rises noticeably to 0.121\,ms and doubling again to 256 pushes it to 0.129\,ms with diminishing MSE gains. 
This confirms that rank 64 provides a favorable accuracy--efficiency trade-off for real tensors.

\paragraph{Synthetic Tensor.}
To simulate the MLP projection setting of Llama3.1-8B, we construct $4096 \times 14336$ synthetic tensors. 
The Column-wise outlier tensor is generated by amplifying random column magnitudes by $25\times$, and the Row-wise outlier tensor is generated by amplifying random column magnitudes by $5\times$ and then transposing, reflecting the typical kurtosis asymmetry between activations and gradients observed in practice. 
We evaluate under the \textbf{RC} pattern pair using five random seeds (42, 256, 1925, 2048, 7777). 
As shown in \Cref{fig:rank_ablation}~(Right), MSE decreases sharply as the rank increases from 0 to 64, dropping from $8.28 \times 10^{-3}$ to $1.00 \times 10^{-3}$, at which point the error plateaus. 
Meanwhile, the latency remains nearly flat between ranks~16 and~64, but increases sharply at rank~128 and rank~256. 
These results are consistent across all five seeds as indicated by the small error bars. 
The synthetic tensor analysis corroborates the real tensor finding: rank~64 achieves the lowest MSE before the latency--accuracy trade-off becomes unfavorable, providing a principled justification for the default $r = 64$ used throughout AdaHOP.

\newpage

\end{document}